\documentclass{article}

\usepackage{times}
\usepackage{graphicx} 
\usepackage{subfigure}

\usepackage{natbib}

\usepackage{algorithm}
\usepackage{algorithmic}

\usepackage{hyperref}

\usepackage[accepted]{icml2018}

\icmltitlerunning{Safe Exploration in Continuous Action Spaces}

\usepackage{amsfonts}
\usepackage{amssymb}
\usepackage{amsmath}
\usepackage{amsthm}

\newtheorem{prop}{Proposition}
\newtheorem{remark}{Remark}

\DeclareMathOperator*{\argmin}{arg\,min}
\DeclareMathOperator*{\argmax}{arg\,max}

\newcommand{\cImm}[2]{c_i({#1},{#2})}
\newcommand{\cNext}[1]{\bar{c}_i({#1})}

\begin{document}

\twocolumn[
\icmltitle{Safe Exploration in Continuous Action Spaces}


%
\begin{icmlauthorlist}
	\icmlauthor{Gal Dalal}{te,go}
	\icmlauthor{Krishnamurthy Dvijotham}{go}
	\icmlauthor{Matej Vecerik}{go}
	\icmlauthor{Todd Hester}{go}
	\icmlauthor{Cosmin Paduraru}{go}
	\icmlauthor{Yuval Tassa}{go} 
\end{icmlauthorlist}

\icmlaffiliation{te}{Technion, Israel Institute of Technology}
\icmlaffiliation{go}{Google DeepMind}

\icmlkeywords{Safe Reinforcement Learning, Deep Reinforcement Learning, Safety Layer, Deep Deterministic Policy Gradient}
\icmlcorrespondingauthor{Gal Dalal}{gald@campus.technion.ac.il}
\vskip 0.3in
]

\printAffiliationsAndNotice{}

\begin{abstract}
	We address the problem of deploying a reinforcement learning (RL) agent on a physical system such as a datacenter cooling unit or robot, where critical constraints must never be violated. We show how to exploit the typically smooth dynamics of these systems and enable RL algorithms to never violate constraints during learning. Our technique is to directly add to the policy a safety layer that analytically solves an action correction formulation per each state. The novelty of obtaining an elegant closed-form solution is attained due to a linearized model, learned on past trajectories consisting of arbitrary actions. This is to mimic the real-world circumstances where data logs were generated with a behavior policy that is implausible to describe mathematically; such cases render the known safety-aware off-policy methods inapplicable. We demonstrate the efficacy of our approach on new representative physics-based environments, and prevail where reward shaping fails by maintaining zero constraint violations.
\end{abstract}

\section{Introduction}
\label{sec:introduction}
	
	In the past two decades, RL has been mainly explored in toy environments \cite{sutton1998reinforcement} and video games \cite{mnih2015human}, where real-world applications were limited to a few typical use-cases such as recommender systems \cite{shani2005mdp}.  However, RL is recently also finding its path into industrial applications in the physical world; e.g., datacenter cooling \cite{evans2016deepmind}, robotics \cite{gu2016deep}, and autonomous vehicles \cite{sallab2017deep}. In all these use-cases, safety is a crucial concern: unless safe operation is addressed thoroughly and ensured from the first moment of deployment, RL is deemed incompatible for them. 
	
	 In real-world applications such as the above, constraints are an integral part of the problem description, and never violating them is often a strict necessity. Therefore, in this work, we define our goal to be maintaining zero-constraint-violations throughout the whole learning process. Note that accomplishing this goal for discrete action spaces is more straightforward than for continuous ones. For instance, one can pre-train constraint-violation classifiers on offline data for pruning unsafe actions. However, in our context, this goal becomes considerably more challenging due to the infinite number of candidate actions. Nevertheless, we indeed manage to accomplish this goal for continuous action spaces and show to never violate constraints throughout the whole learning process.
	
	Specifically, we tackle the problem of safe control in physical systems, where certain observable quantities are to be kept constrained. To illustrate, in the case of datacenter cooling, temperatures and pressures are to be kept below respective thresholds at all times; a robot must not exceed limits on angles and torques; and an autonomous vehicle must always maintain its distance from obstacles above some margin. In this work, we denote these quantities as \emph{safety signals}. As these are physical quantities, we exploit their smoothness for avoiding unexpected, unsafe operation.  Moreover, we deal with the common situation where offline logged data are available; it thus can be used to pre-train models for aiding safety from the initial RL deployment moments.
	
	Safe exploration, as depicted above, traditionally requires access to data generated with some known behavior policy upon which gradual safe updates are performed; see \cite{thomas2015safe} for a comprehensive study of such off-policy methods. Such data are necessary because, unless assumed otherwise, actions in one state might have catastrophic consequences down the road. Hence, long-term behavior should be inferred in advance prior to deployment. In contrast, in this work, we eliminate the need for behavior-policy knowledge as we focus on physical systems, whose actions have relatively short-term consequences. Obviating behavior-policy knowledge is a key benefit in our work, as lack of such data is a challenging yet familiar real-world situation.
	It is rarely the case that past trajectories in complex systems were generated using a consistent behavior policy that can be mathematically described. Such systems are traditionally controlled by humans or sophisticated software whose logic is hard to portray. Hence, off-policy RL methods are deemed inapplicable in such situations. Contrarily, we show how single-step transition data can be efficiently exploited for ensuring safety. To demonstrate our method's independence of a behavior policy, in our experiments we generate our pre-training data with purely random actions.
	
	Our approach relies on one-time initial pre-training of a model that predicts the change in the safety signal over a single time step. This model's strength stems from its simplicity: it is a first-order approximation with respect to the action, where its coefficients are the outputs of a state-fed neural network (NN).  We then utilize this model in a safety layer that is composed directly on top the agent's policy to correct the action if needed; i.e., after every policy query, it solves an optimization problem for finding the minimal change to the action such that the safety constraints are met. Thanks to the linearity with respect to actions, the solution can be derived analytically in closed-form and amounts to basic arithmetic operations. Thus, our safety layer is both differentiable and has a trivial three-line software implementation. Note that relating to our safety mechanism as a `safety layer' is purely a semantical choice; it merely is a simple calculation that is not limited to the nowadays popular deep policy networks and can be applied to any continuous-control algorithm (not necessarily RL-based).
	
	\section{Related Work}
	\label{sec:related work}
	
		As this work focuses on control problems with continuous state and action spaces, we limit our comparison to the literature on safe RL in the context of policy optimization that attempts to maintain safety also during the learning process. Such an example is \cite{achiam2017constrained}, where constrained policy optimization was solved with a modified trust-region policy gradient. There, the algorithm's update rule projected the policy to a safe feasibility set in each iteration. Under some policy regularity assumptions, it was shown to keep the policy within constraints in expectation. As such, it is unsuitable to our use-cases, where safety must be ensured for all visited states.
		Another recent work \cite{berkenkamp2017safe} described control-theoretic conditions under which safe operation can be guaranteed for a discretized deterministic control framework. An appropriate Lyapunov function was identified for policy attraction regions if certain Lipschitz continuity conditions hold. Though under the appropriate conditions safe exploration was guaranteed, knowledge on the specific system was required. Moreover, a NN may not be Lipschitz continuous with a reasonable coefficient.
		Lastly, very recent work \cite{pham2017optlayer} utilized an in-graph QP solver first introduced in \cite{amos2017optnet}. It exhibited an approach similar in nature to the one we take: solve an optimization problem in the policy-level, for ensuring safety on a state-wise basis. However, two main differences are to be noticed. First, the solution there relied on an in-graph implementation of a full QP solver \cite{amos2017optnet}, which runs an iterative interior-point algorithm with each forward propagation. This is both a challenge to implement (currently only a pytorch version is available), and computationally expensive. An additional difference from our work is that in \cite{pham2017optlayer} expert knowledge was required to explicitly hand-design the physical constraints of a robotic arm. Contrarily, in our method no such manual effort is needed; these dynamics are learned directly from data, while also being behavior-policy independent.
		
		To summarize, this work is the first, to our knowledge, to solve the problem of state-wise safety directly at the policy level, while also doing it in a data-driven fashion using arbitrary data logs. Moreover, it can be applied to any continuous-control algorithm; it is not restricted to a specific RL algorithm or any at all.

	\section{Definitions}
	We consider a special case of constrained Markov decision processes (CMDP) \cite{altman1999constrained}, where observed \emph{safety signals} should be kept bounded. First, let us denote by $[K]$ the set $\{1,\dots,K\},$ and by $\left[x\right]^+$ the operation $\max\{x,0\},$ where $x\in \mathbb{R}.$ A CMDP is a tuple $(\mathcal{S},\mathcal{A},P,R,\gamma,\mathcal{C})$, where $\mathcal{S}$ is a state space, $\mathcal{A}$ is an action space, $P: \mathcal{S} \times \mathcal{A} \times \mathcal{S} \rightarrow [0,1]$ is a transition kernel, $R: \mathcal{S} \times \mathcal{A} \rightarrow \mathbb{R}$ is a reward function, $\gamma\in(0,1)$ is a discount factor, and $\mathcal{C} = \{c_i: \mathcal{S} \times \mathcal{A} \rightarrow \mathbb{R}\mid i\in [K]\}$ is a set of \emph{immediate-constraint} functions. Based on that, we also define a set of safety signals $\bar{\mathcal{C}} = \{\bar{c}_i: \mathcal{S} \rightarrow \mathbb{R}\mid i\in [K]\}.$ 
	These are per-state observations of the immediate-constraint values, which we introduce for later ease of notation. To illustrate, if $c_1(s,a)$ is the temperature in a datacenter to be sensed after choosing $a$ in $s$, $\bar{c}_1(s')$ is the same temperature sensed in $s'$ after transitioning to it. 
	In the type of systems tackled in this work, $P$ is deterministic and determines $f$ s.t. $s'=f(s,a)$. Thus, we have $\bar{c}_i(s') \triangleq c_i(s,a).$ For a general non-deterministic transition kernel, $\bar{c}_i(s')$ can be defined as the expectation over $s' \sim P(\cdot|s,a).$ 
	Lastly, let policy $\mu:{\cal S} \rightarrow \mathcal{A}$
	 be a stationary mapping from states to actions.
	
	\section{State-wise Constrained Policy Optimization}
	We study safe exploration in the context of policy optimization, where at each state, all safety signals $\cNext{\cdot}$ are upper bounded by corresponding constants $C_i \in \mathbb{R}$:
	\begin{align}
	\label{eq:constrained policy optimization}
		&\max_\theta \mathbb{E}[\sum_{t=0}^\infty \gamma^t R(s_t,\mu_\theta(s_t))] \\
		&~\mbox{s.t.} \quad \cNext{s_t}\leq C_i ~~ \forall i \in [K] \enspace, \nonumber
	\end{align}
	 where $\mu_\theta$ is a parametrized policy.
	 
	 We stress that our goal is to ensure state-wise constraints not only for the solution of \eqref{eq:constrained policy optimization}, but also for its optimization process. This goal might be intractable in general since for an arbitrary MDP some actions can have a long-term effect in terms of possible state paths. However, for the types of physical systems we consider it is indeed plausible that safety constraints can be ensured by adjusting the action in a single (or few) time step(s). In the context of our real-world use-cases, cooling system dynamics are governed by factors such as the first-order heat-transfer  differential equation \cite{goodwine2010engineering}, and the second-order Newton differential equation that governs the water mass transfer. The latter also governs the movement of a robotic arm or a vehicle on which one applies forces. In these types of control problems, it is feasible to satisfy state-wise constraints even in the presence of inertia given reasonable slack in the choice of $C_i$. We expand on this further and provide evidence in Section~\ref{sec:experiments}. 

	\section{Linear Safety-Signal Model}
	\label{sec:linearized model}
	
	Solving \eqref{eq:constrained policy optimization} is a difficult task, even for the types of systems listed above. A major contributor to this challenge is the RL agent's intrinsic need  to explore for finding new and improved actions. Without prior knowledge on its environment, an RL agent initialized with a random policy cannot ensure per-state constraint satisfaction during the initial training stages. This statement also holds when the reward is carefully shaped to penalize undesired states: for an RL agent to learn to avoid undesired behavior it will have to violate the constraints enough times for the negative effect to propagate in our dynamic programming scheme.
	
	In this work, we thus incorporate some basic form of prior knowledge, based on single-step dynamics. Single-step transition data in logs is rather common, and, as explained before, more realistic compared to also knowing behavior policies. We do not attempt to learn the full transition model, but solely the immediate-constraint functions $\cImm{s}{a}.$	
	While it is attractive to simply approximate them with NNs that take $(s,a)$ as inputs, we choose a more elegant approach that comes with significant advantages listed in Subsection~\ref{sec: additional loss}. Namely, we perform the following linearization:
	  \[
	  	\cNext{s'} \triangleq \cImm{s}{a} \approx \cNext{s} + g(s;w_i)^\top a,
	  \]
	   where $w_i$ are weights of a NN, $g(s;w_i),$ that takes $s$ as input and outputs a vector of the same dimension as $a$. This model is a first-order approximation to $\cImm{s}{a}$ with respect to $a;$ i.e., an explicit representation of sensitivity of changes in the safety signal to the action using features of the state. See Fig.~\ref{fig:linearization} for a visualization.
	   
	   \begin{remark}
	   	 Linear approximations of non-linear physical systems prove accurate and are well accepted in many fields, e.g. aircraft design \cite{liang2013nonlinear}. For a comprehensive study see \cite{enqvist2005linear}.
	   \end{remark}
	   \begin{figure}
	   		\begin{center}
	   			\includegraphics[scale=0.63]{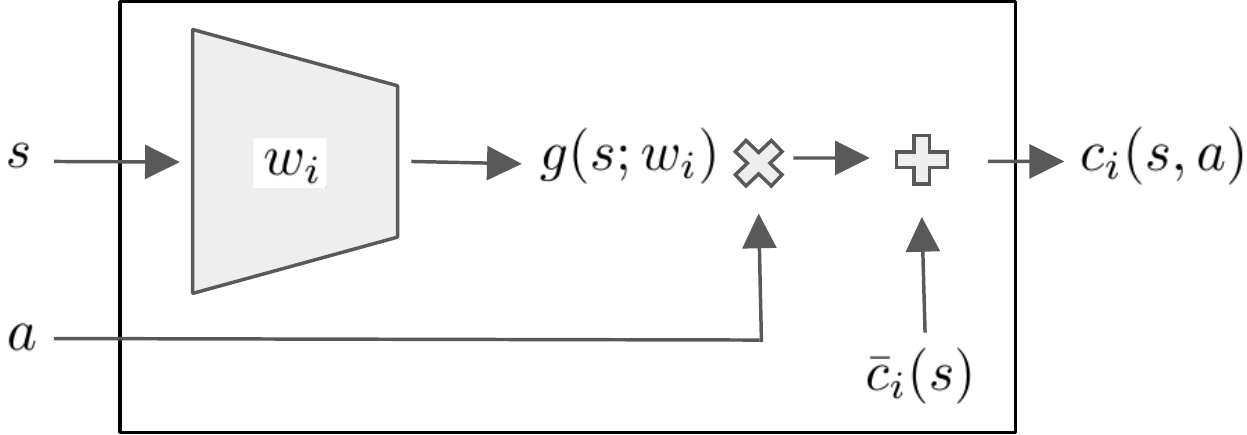}
	   		\end{center}
   			\caption{Each safety signal $\cImm{s}{a}$ is approximated with a linear model with respect to $a$, whose coefficients are features of $s,$ extracted with a NN.}
  				\label{fig:linearization}
	   \end{figure}
	   
	   Given a policy-oblivious set of tuples $D=\{(s_j,a_j,s_j')\}$, we train $g(s;w_i)$ by solving
	   \begin{align}
	   \label{eq:pre-training}
	   		\argmin_{w_i} \sum_{(s,a,s') \in D}\left(\cNext{s'} - (\cNext{s} + g(s;w_i)^\top a)\right)^2,
	   \end{align}
	   where $\cNext{s}$ is assumed to be included in $s$. 
	   
	  In our experiments, to generate $D$  we merely initialize the agent in a uniformly random location and let it perform uniformly random actions for multiple episodes. The episodes terminate when a time limit is reached or upon constraint violation. The latter corresponds to real-world mechanisms governing production systems: high-performance, efficient control is often backed up by an inefficient conservative policy; when unsafe operation flags are raised, a preemption mechanism is triggered and the conservative policy kicks in to ensure constraint satisfaction.
	  
	  Training $g(s;w_i)$ on $D$ is performed once per task as a pre-training phase that precedes the RL training. However, additional continual training of $g(s;w_i)$ during RL training is also optional. Since in our experiments continual training showed no benefit compared to solely pre-training, we only show results of the latter.
	   
	  \section{Safety Layer via Analytical Optimization }
	  \label{sec:optimization layer}
	  
	  We now show how to solve problem \eqref{eq:constrained policy optimization} using the policy gradient algorithm \cite{baxter2001infinite} via a simple addition to the policy itself. We experiment with Deep Deterministic Policy Gradient (DDPG) \cite{lillicrap2015continuous} whose policy network directly outputs actions and not their probabilities. However, our approach is not limited to it and can be added, as is, to probabilistic policy gradient or any other continuous-control algorithm.
	  
	   Denote by $\mu_\theta(s)$ the deterministic action selected by the deep policy network. Then, on top of the policy network we compose an additional, last layer, whose role is to solve 
	    \begin{align} 
	   \label{eq:optimization layer}
	   &\argmin_{a}\tfrac{1}{2}\|a-\mu_\theta(s)\|^2 \\
	   &~\mbox{s.t.}~ \cImm{s}{a} \leq C_i ~~\forall i\in [K] \enspace. \nonumber
	   \end{align}
	   This layer, which we refer to as \emph{safety layer}, perturbs the original action as little as possible in the Euclidean norm in order to satisfy the necessary constraints. Fig.~\ref{fig:safety layer} visualizes its relation to the policy network.
	   
	    \begin{figure}
	   	\begin{center}
	   		\includegraphics[scale=0.43]{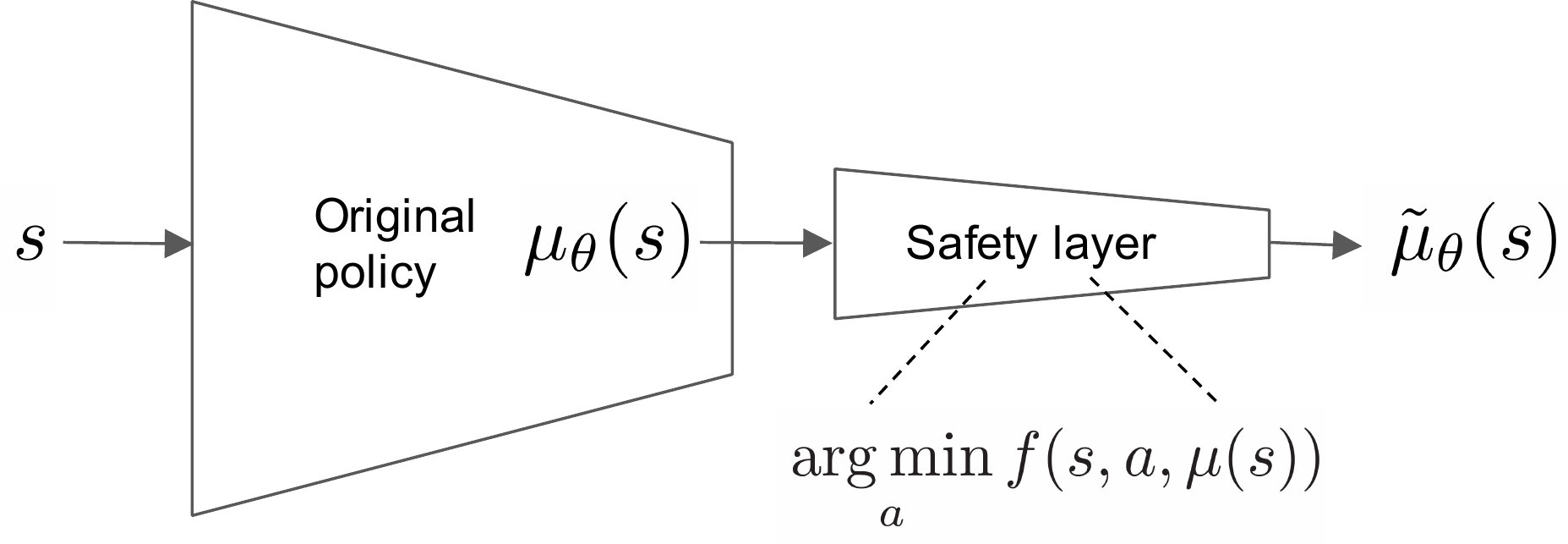}
	   	\end{center}
	   	\caption{A safety layer is composed on top of a deep policy network, solving an action correction optimization program with each forward-propagation. Our linearized safety-signal model allows a closed-form solution $\tilde{\mu}_\theta(s) = \argmin_a f(s,a,\mu(s))$ that reduces to a trivial linear projection.}
	   	\label{fig:safety layer}
	   \end{figure}

	   
	   To solve \eqref{eq:optimization layer} we now substitute our linear model for $\cImm{s}{a},$ introduced in Section~\ref{sec:linearized model}, and obtain the quadratic program
	   \begin{align} 
	   \label{eq:linearized optimization layer}
	   a^*=&\argmin_{a}\tfrac{1}{2}\|a-\mu_\theta(s)\|^2 \\
	   &~\mbox{s.t.}~ \cNext{s} + g(s;w_i)^\top a \leq C_i ~~ \forall i\in [K]\enspace. \nonumber
	   \end{align}
	   
	   Thanks to the positive-definite quadratic objective and linear constraints, we can now find the global solution to this convex problem. Generally, to solve it one can implement an in-graph iterative QP-solver such as in \cite{amos2017optnet}. This would result in a method similar to the one in \cite{pham2017optlayer}, but with the advantage that all the physical constraint model is learned directly from data instead of being hand-designed. Alternatively, if the number of active constraints is known to be bounded by some $m\leq K$, one can exhaustively iterate on all $\binom{K}{m}$ combinations of possibly active constraints and select the optimal feasible one; this would be reasonable for a small $m$. 
	   
	    However, in this work, at the expense of one simplifying assumption, we gain the  benefit of obtaining a closed-form analytical solution to \eqref{eq:linearized optimization layer} that has a trivial three-lines-of-code software implementation. The assumption is that no more than a single constraint is active at a time. As demonstrated in our experiments, this is reasonable to assume when an agent navigates in a physical domain and avoids obstacles. As the distance from each obstacle is modeled as a separate constraint, only a single obstacle is the closest one at a time. Proximity to corners shows to pose no issues, as can be seen in the plots and videos in Section~\ref{sec:experiments}. Moreover, for other systems with multiple intersecting constraints, a joint model can be learned. For instance, instead of treating distances from two walls as two constraints, the minimum between them can be treated as a single constraint. In rudimentary experiments, this method produced similar results to not using it, when we grouped two constraints into one, in the first and simplest task in Section~\ref{sec:experiments}. However, jointly modeling the dynamics of more than a few safety signals with a single $g(\cdot;\cdot)$ network is a topic requiring careful attention, which we leave for future work.
	    
	    We now provide the closed-form solution to \eqref{eq:linearized optimization layer}.		
	  \begin{prop}
	  	\label{prop:analytical solution}
	  	Assume there exists a feasible solution to \eqref{eq:linearized optimization layer} denoted by $(a^*,\{\lambda^*_i\}_{i=1}^K)$, where $\lambda^*_i$ is the optimal Lagrange multiplier associated with the $i$-th  constraint. Also, assume $|\{ i | \lambda^*_i > 0 \}| \leq 1;$ i.e., at most one constraint is active. Then 
	  	\begin{equation}
	  	\label{eq:optimal multiplier}
	  		\lambda^*_i = \left[\frac{g(s;w_i)^\top \mu_\theta(s) + \cNext{s} - C_i}{g(s;w_i)^\top g(s;w_i)}\right]^+
	  	\end{equation}
	  	and
	  	\begin{equation}
	  	\label{eq:optimal action}
	  		a^* = \mu_\theta(s) - \lambda_{i^*}^*g(s;w_{i^*}),
	  	\end{equation}
	  	where $i^* = \argmax_i \lambda_i^*.$
	  \end{prop}
  \begin{proof}
  	As the objective function and constraints in \eqref{eq:linearized optimization layer} are convex, a sufficient condition for optimality of a feasible solution $(a^*,\{\lambda^*_i\}_{i=1}^K)$ is for it to satisfy the KKT conditions. The Lagrangian of \eqref{eq:linearized optimization layer} is
  	\begin{align*}
  	L(a,&\lambda) = \\
  	&\tfrac{1}{2}\|a-\mu_\theta(s)\|^2 + \sum_{i=1}^K \lambda_i \left(\cNext{s} + g(s;w_i)^\top a - C_i\right);
  	\end{align*}
  	and hence the KKT conditions at $(a^*,\{\lambda^*_i\}_{i=1}^K)$ are
  	\begin{align}
  		\nabla_a L =  a^*-\mu_\theta(s) + \sum_{i=1}^K \lambda^*_i g(s;w_i) &= \mathbf{0}, \label{eq:Lagrangigan gradient} \\
  		\lambda_i^* \left(\cNext{s} + g(s;w_i)^\top a^* - C_i\right) &= 0~ \forall i\in [K]~. \label{eq:comp slackness}
  	\end{align}
  	 First, consider the case where $|\{ i | \lambda^*_i > 0 \}| = 1,$ i.e., $\lambda^*_{i^*}>0.$ We then easily get \eqref{eq:optimal action} from \eqref{eq:Lagrangigan gradient}.
  	 Next, from \eqref{eq:comp slackness} we have that  $\bar{c}_{i^*}(s) + g(s;w_{i^*})^\top a^* - C_{i^*}=0.$ Substituting \eqref{eq:optimal action} in the latter gives $\lambda^*_{i^*}g(s;w_{i^*})^\top g(s;w_{i^*}) = g(s;w_{i^*})^\top \mu_\theta(s) + \bar{c}_{i^*}(s) - C_{i^*}.$ This gives us \eqref{eq:optimal multiplier} when $i=i^*$. As for $i \in [K] \setminus \{i^*\},$ the corresponding constraints are inactive since $\lambda^*_i=0.$ Hence,  $\cNext{s} + g(s;w_i)^\top a^* - C_i < 0,$ making the fraction in \eqref{eq:optimal multiplier} negative, which indeed results in a value of $0$ due to the $[\cdot]^+$ operator and gives us \eqref{eq:optimal multiplier} also when $i \in [K] \setminus \{i^*\}.$
  	 
  	 To conclude the proof, consider the second case where $\lambda_i^* = 0 ~\forall i \in [K].$ From \eqref{eq:Lagrangigan gradient} we have that $a^*=\mu_\theta(s).$ This gives us \eqref{eq:optimal action} since $\lambda^*_{i^*}=0.$ Lastly, \eqref{eq:optimal multiplier} holds due to the same inactive constraints argument as above, this time uniformly $\forall i \in [K].$
  \end{proof}

	The solution \eqref{eq:optimal action} is essentially a linear projection of the original action $\mu_\theta(s)$ to the ``safe'' hyperplane with slope $g(s;w_{i^*})$ and intercept $\bar{c}_{i^*}(s) - C_{i^*}$. In terms of implementation, it consists of a few primitive arithmetic operations: vector products followed by a 'max' operation. The benefits of its simplicity are three-fold: i) it has a trivial, almost effortless software implementation; ii) its computational cost is negligible; and iii) it is differentiable (almost everywhere, as is ReLu).
	   \subsection{An Alternative: Additional Loss Term} 
	   \label{sec: additional loss}
	   
	   To stress the prominence of our linear model in solving \eqref{eq:optimization layer}, we now briefly describe the drawbacks of an alternative likely choice we initially experimented with: a straight-forward $(s,a)$-fed NN model for $\cImm{s}{a}$. In this case, an approximate solution to \eqref{eq:optimization layer} can be obtained by penalizing the objective for constraint violations and solving the unconstrained surrogate
	   \begin{align} 
	   \label{eq:penalized optimization layer}
	   \argmin_{a}\bigg\{\tfrac{1}{2}\|a-\mu_\theta(s)\|^2 + \sum_{i=1}^K \lambda_i \left[\cImm{s}{a} - C_i\right]^+\bigg\},
	   \end{align}
	   where $\{\lambda_i>0\}$ are now hyper-parameters. Problem~\eqref{eq:penalized optimization layer} can be solved numerically using gradient descent. However, even though from our experience this approach indeed works (corrects actions to safe ones), it is inferior to our analytic approach for several reasons:
	   \begin{enumerate}
	   	\item Running  gradient descent with every policy query (i.e. forward propagation) requires sophisticated in-graph implementation and is computationally intensive.
	   	\item Since the sensitivity of $\cImm{s}{a}$ varies for different entries in $a$, different orders of magnitude are observed in entries of the resulting gradient $\nabla_a\cImm{s}{a}$. This causes numerical instabilities and long convergence times, and requires careful stepsize selection.
	   	\item Solutions to the non-convex \eqref{eq:penalized optimization layer} are local minima, which depend on non-reliable convergence of an iterative optimization algorithm, as opposed to the closed-form  global optimum we obtain for \eqref{eq:linearized optimization layer}.
	   	\item There are $K$ hyper-parameters necessitating tuning.
	   \end{enumerate}
   
   \section{Experiments}
   \label{sec:experiments}
   Per each task introduced next, we run the initial pre-training phase described in Section~\ref{sec:linearized model}. We construct $D$ with 1000 random-action episodes per task. We then add our pre-trained safety layer to the policy network. As mentioned earlier, our RL algorithm of choice for the experiments is DDPG \cite{lillicrap2015continuous}. In this section we show that during training, DDPG never violates constraints and converges faster compared to without our addition.
   
   To mimic the physics-based use-cases described in this work, where continuous safety signals are observations of the state ought to be constrained, we set up appropriate simulation domains in Mujoco \cite{todorov2012mujoco}. In these domains, an object is located in some feasible bounded region. Each of the constraints, therefore, lower bounds the object's distance to each of the few boundaries. Though the lower bound on the distance is zero by the tasks' definition, in practice we set it to be some small positive value to allow slack for avoidance actions in the presence of inertia. In all simulations, the episode immediately terminates in the case of a constraint violation. These conditions comply with the ones in our real-world examples: a datacenter cooling system's maximal temperature set for the formulation and algorithm need not be the one originally defined by the datacetner operator; a lower, more conservative value can be set to allow for some slack. Episode termination corresponds to preemption, followed by a swap to some backup heuristic each time the conservative temperature cap is reached.

   We now introduce our two new Mujoco domains: Ball and Spaceship, consisting of two tasks each. The dynamics of Ball and Spaceship are governed by first and second order differential equations, respectively. As such, they are representative domains to the systems of interest in this paper.
   
  \subsection{Ball Domain}
  In the Ball domain, the goal is to bring a ball as close as possible to a changing target location, by directly setting the velocity of the ball every $4$-th time-step (it is common for the operational frequency of a typical torque controller to be less than the environment's one). Each episode lasts at most $30$ seconds, during which the target is  appearing in a new, uniformly random location every $2$ seconds. Let us now define the $d$-dimensional $[a,b]$-cube $B_{[a,b]}^{(d)}=\{x|a\leq x_i \leq b,~i=1,\dots,d\}.$  The feasible region for the ball is $B_{[0,1]}^{(d)},$ while for the target it is $B_{[0.2,0.8]}^{(d)}.$ If the ball steps out of $B_{[0,1]}^{(d)},$ the episode terminates. To allow some slack in maneuvering away from the boundaries, we set $C_i$ in  \eqref{eq:optimization layer} so as to constrain the ball's feasible region to effectively be $B_{[0.1,0.9]}^{(d)}.$ This is for the safety layer to start correcting the actions once the ball steps out of the latter cube. Let the ball location, ball velocity and target location respectively be $x_B,v_B,x_T \in \mathbb{R}^d.$ Then, $s=(x_B,v_B,x_T + \epsilon_d),$ where $\epsilon_d \sim \mathcal{N}(0,0.05 \cdot I^{d\times d});$ i.e., only a noisy estimate of the target is observed. The action $a=v_B$; i.e., it is to set the velocity of the ball. Actions are taken every $4$ time-steps and remain constant in between. The dynamics are governed by Newton's laws, with a small amount of damping. The reward has a maximum of $1$ when the ball is exactly at the target and quickly diminishes to $0$ away from it: $R(s,a)=\left[1-10\cdot\|x_B - x_T\|_2^2\right]^+.$ Lastly, $\gamma=0.99$. Our experiments are conducted on two tasks: Ball-1D where $d=1$, Ball-3D where $d=3.$ Images of the two can be found in Fig.~\ref{fig:ball domain}.
  
  As the ball's velocity is controlled directly, its dynamics are governed by a first-order differential equation. This domain thus represents phenomena such as heat transition and several known control tasks such as maintaining the speed of a rotating engine, whether it is a classic steam engine or cruise-control in modern vehicles  \cite{sotomayor2006stability}.

	\begin{figure}
		\centering
		\includegraphics[scale=0.29]{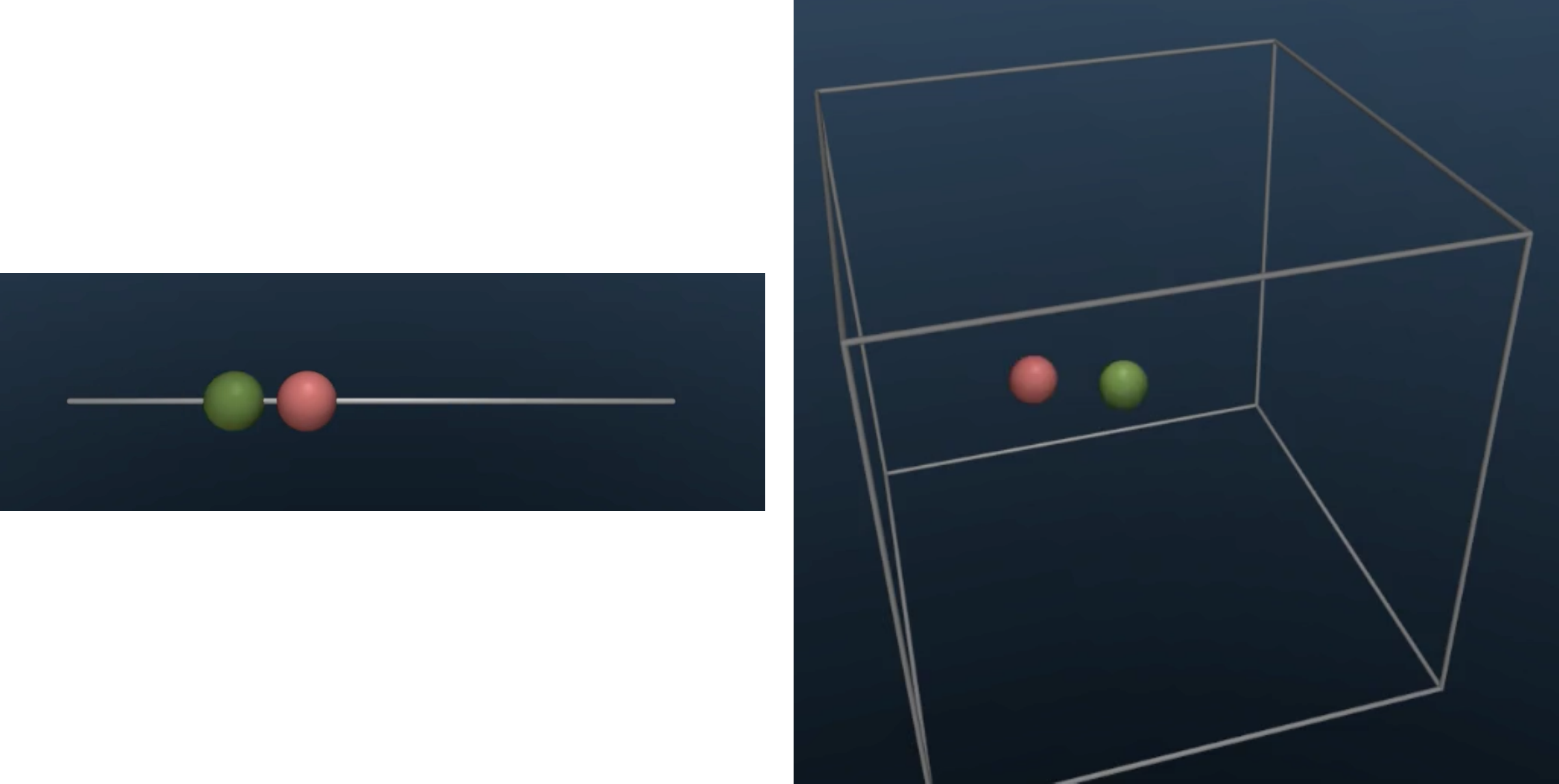}
		\caption{Ball-1D (left) and Ball-3D (right) tasks. The goal is to keep the green ball as close as possible to the pink target ball by setting its velocity. The safe region is the $[0,1]$ interval in Ball-1D and the $[0,1]$ cube in Ball-3D; if the green ball steps out of it, the episode terminates.}
		\label{fig:ball domain}
		   	\vspace{-0.4cm}
	\end{figure}

  \subsection{Spaceship Domain}
  In the Spaceship domain, the goal is to bring a spaceship to a fixed target location by controlling its thrust engines. Hence, as opposed to setting velocities in the Ball domain, here we set the forces. Our first task for this domain is Spaceship-Corridor, where the safe region is bounded between two infinite parallel walls. Our second task, Spaceship-Arena, differs from the first in the shape of the safe region; it is bounded by four walls in a diamond form. Images of the two tasks are given in Fig.~\ref{fig:spaceship domain}. Episodes terminate when one of the three events occur: the target is reached, the spaceship's bow touches a wall, or the time limit is reached. Time limits are $15$ seconds for Corridor and $45$ seconds for Arena. Relating to the screenshots in Fig.~\ref{fig:spaceship domain}, the spaceship's initialization location is uniformly random in the lowest third part of the screen for Corridor, and the right-most third part of the screen for Arena.
   
   \begin{figure}
   	\centering
   	\includegraphics[scale=0.26]{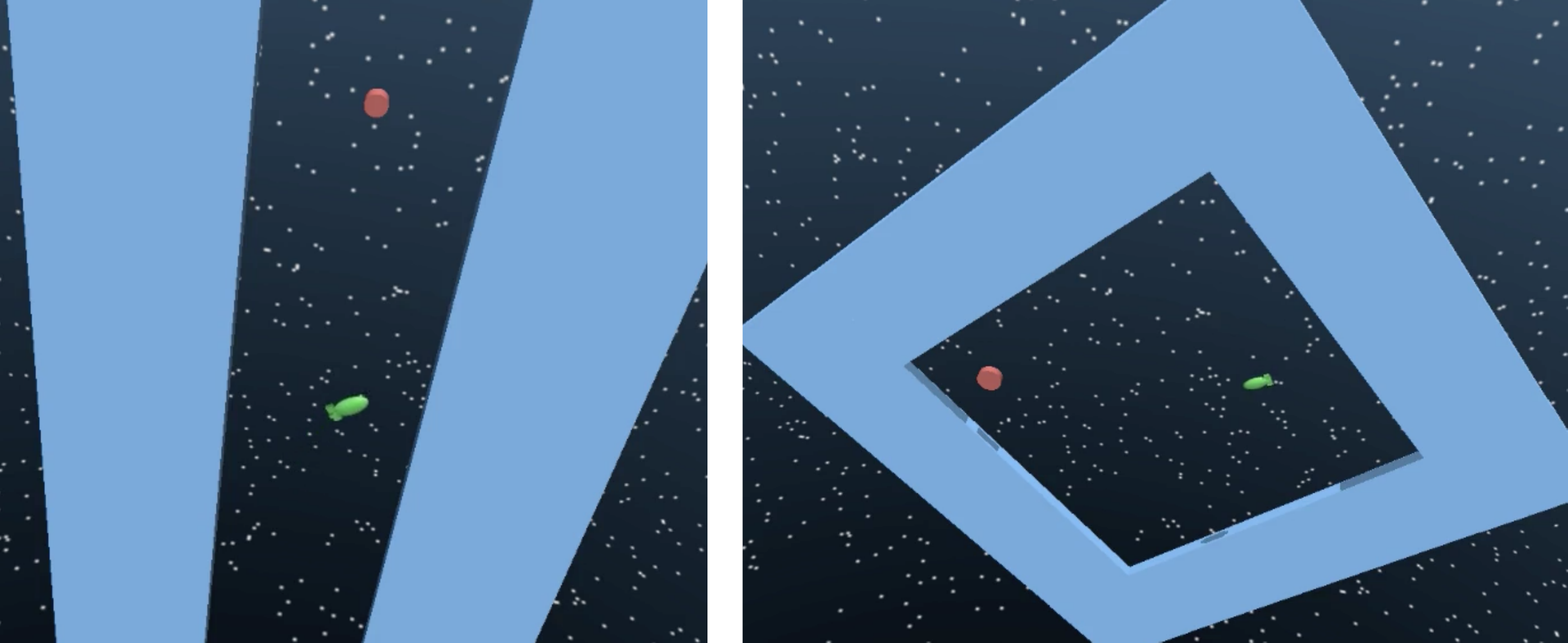}
   	\caption{Spaceship-Corridor (left) and Spaceship-Arena (right) tasks. The goal is to bring the green spaceship to to the pink rounded target by controlling its thrust engines. Touching the walls with the spaceship's bow terminates the episode.}
   	\label{fig:spaceship domain}
   	\vspace{-0.3cm}
   \end{figure}

  In this domain, the state is the spaceship's location and velocities; the action $a \in [-1,1]^2$ actuates two thrust engines in forward/backward and right/left directions; the transitions are governed by the rules of physics where damping is applied; the reward is sparse: 1000 points are obtained when reaching the target and 0 elsewhere; and $\gamma=0.99.$ As in the Ball domain, a small gap away from each wall is incorporated in the choice of $C_i$ in  \eqref{eq:optimization layer}. This is for the safety layer to begin correcting the actions a few time-steps before the spaceship actually reaches a wall with its bow. For both tasks, this gap is $0.05,$ where, for comparison, the distance between the walls in Corridor task is $1.$
  
 Since the control is via forces, the spaceship's dynamics are governed by a second-order differential equation. This domain thus represents situations such as pumping masses of water for cooling, and actuating objects such as robotic arms.
 
\begin{figure*}
	\begin{center}
		\includegraphics[scale=0.3]{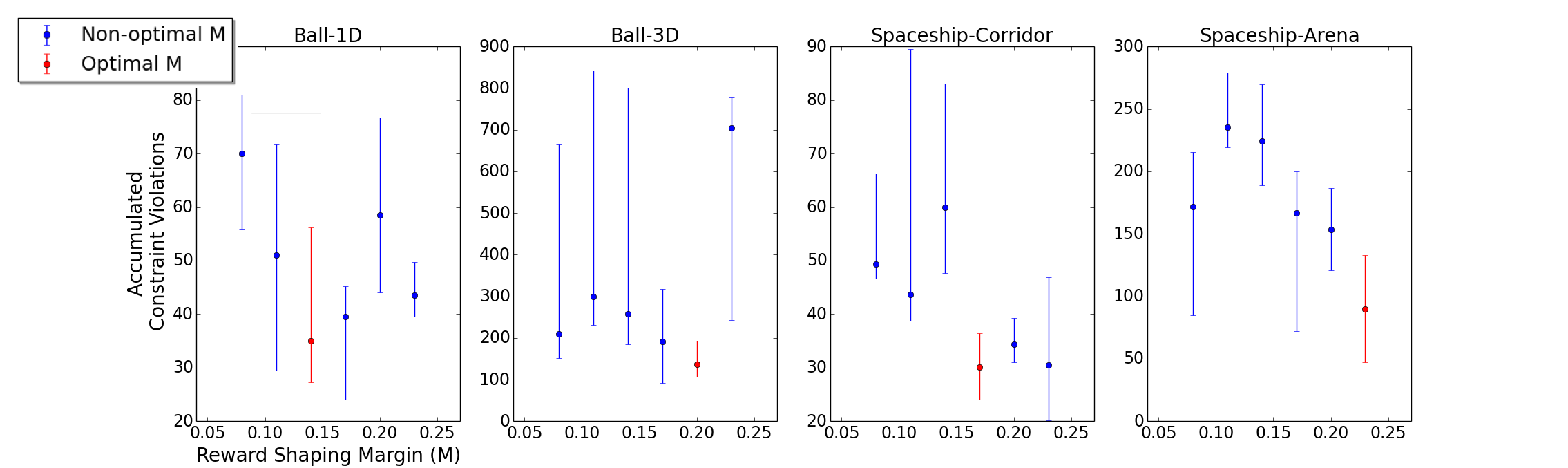}
	\end{center}
	\caption{Accumulated constraint violations (lower is better) throughout the training of DDPG+reward shaping, per each task. Plotted are medians with upper and lower quantiles of 10 seeds. The $x$-axis corresponds to different choice of $M$ -- the margin from limits in which reward penalty is incurred. The optimal choice of $M$ is colored red.}
	\label{fig:reward shaping}
\end{figure*}

   \subsection{Implementation Details}
   Per each of the four tasks above, we pre-train the safety-signal model for each of the task's constraints; i.e., we solve \eqref{eq:pre-training} for each $i\in [K]$ using that task's data-set $D.$ In a hyper-parameter sweep, we witnessed that the ability to achieve our end-goal is relatively insensitive to the architecture of $g(\cdot;w_i)$ and a simple one suffices. Namely, for achieving zero-constraint-violations after plugging $g(\cdot;w_i)$ into the safety layer, a single-layered NN with $10$ hidden neurons was adequate for all tasks. For training we use Adam optimizer \cite{kingma2014adam} with a batch size of $256$. We also report that low validation errors for $g(\cdot;w_i)$ don't necessarily correspond to a better safety layer; it is the exploratory nature of $D$ that matters. Specifically, we witnessed the following counter-intuitive behavior: generating $D$ with a policy that fixes a single randomly drawn action for each episode, as opposed to a random action in each step, yields better regression accuracy for $g(\cdot;w_i)$ but lesser safety layer performance. 
   
    Based on a hyper-parameter sweep for DDPG, for the actor and critic we use two-hidden-layer NNs of respective sizes $(100,100)$ and $(500,500).$ The rest of the experiment parameters were taken to be as in \cite{lillicrap2015continuous}.
   
   \subsection{Safety Layer versus Reward Shaping}
   Before exhibiting the performance of our safety layer, we first relate to a natural alternative approach for ensuring safety: manipulate the agent to avoid undesired areas by artificially shaping the reward. This can be done by setting the reward to large negative values in subsets of the state-space. In our case, those areas are the neighborhood of the enclosing boundaries. Therefore, for comparison, we first ran a series of such reward shaping experiments on DDPG without the safety layer. We set the penalty to be on the same scale as the original reward: $-1$ in the Ball domain and $-1000$ in the spaceship domain. The margin from boundaries in which the shaping occurs is a parameter $M$ which we cross-validated per each task. We experimented with values of $M \in \{0.08,0.11,0.14,0.17,0.2,0.23\}.$  Per each $M$, we ran DDPG, where each training episode is followed by an evaluation episode for which we count whether it terminated due to a constraint violation. Fig.~\ref{fig:reward shaping} gives, per each $M,$ the median with upper and lower quantiles of accumulated constraint violations of $10$ seeds of DDPG runs. Per each task, we mark the best choice of $M$ in red. 
   
   Fig.~\ref{fig:reward shaping} depicts the drawbacks of reward shaping for ensuring safety. The first fault is the failure to achieve our zero-constraint-violations goal; all parameter choices resulted in significant portions of the episodes terminating due to violations. The second drawback is the difficulty of choosing the correct parameter. As seen in the figure, there is no clear trend for the dependence on $M$; each task has a different 'best value', and the plots have no structure.

   Next, we compare the performance of our safety layer to that of the best reward shaping choice, and to no reward shaping at all. Namely, we compare the following alternatives: DDPG, DDPG+reward shaping, and DDPG+safety layer.  For reward shaping we use the same simulation results of the best $M$ (colored red in Fig.~\ref{fig:reward shaping}). The comparison outcomes are summarized in Fig.~\ref{fig:safety comparison}. Its top row gives the sum of discounted rewards per evaluation episode, and the bottom provides the number of constraint violations, accumulated over all evaluation episodes. The plots show medians along with upper and lower quantiles of $10$ seeds.
   \begin{figure*}
   	\begin{center}
   		\includegraphics[scale=0.3]{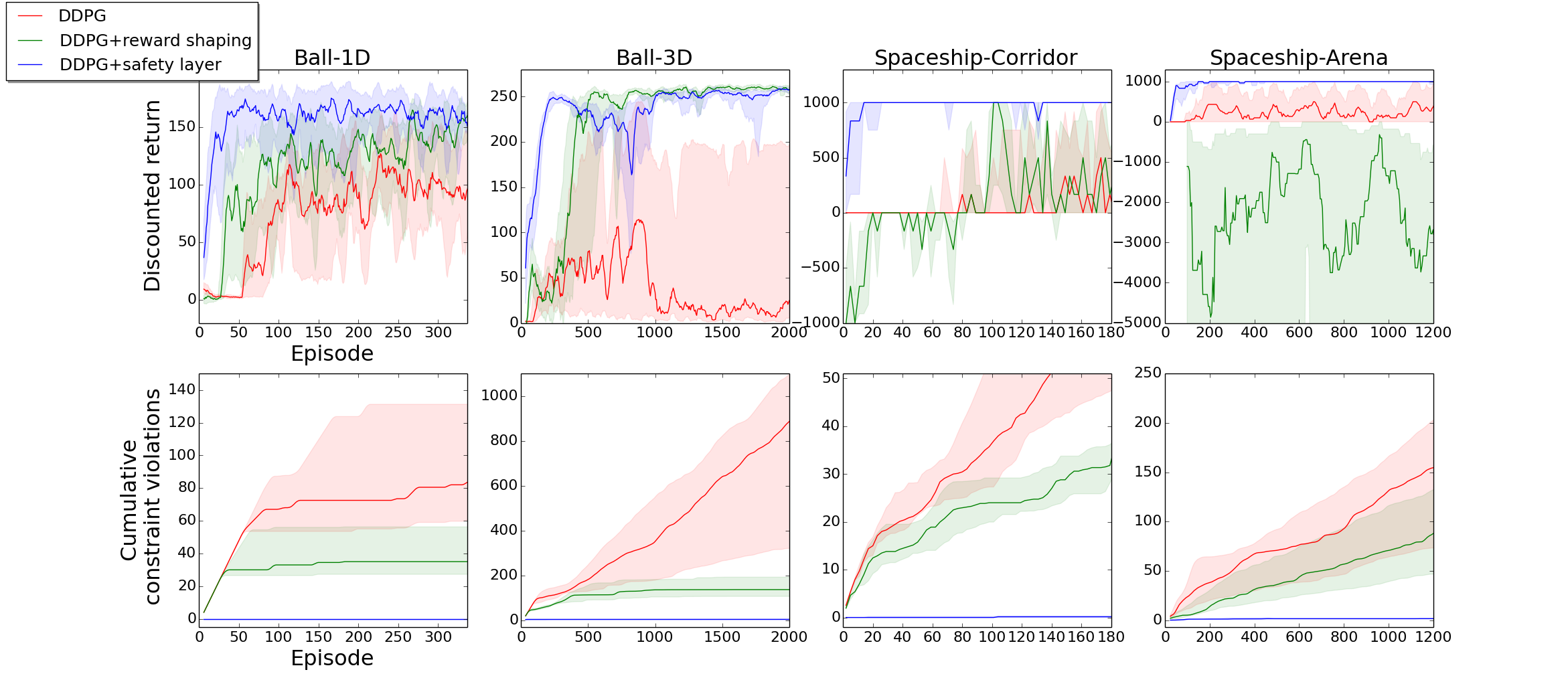}
   	\end{center}
   	\caption{Convergence and constraint violations of DDPG, DDPG+reward shaping, and DDPG+safety layer, per each task. Plotted are medians with upper and lower quantiles of 10 seeds. The top row shows the sum of discounted rewards, where the bottom gives the cumulative number of episodes terminated due to a constraint violation. As shown by the bottom blue curves, with the safety layer, constraints were never violated.}
   	\label{fig:safety comparison}
   \end{figure*}

   The most prominent insight from Fig.~\ref{fig:safety comparison} is that the constraints were never violated with the safety layer. This is true for all $10$ seeds of each of the four tasks. Secondly, the safety layer dramatically expedited convergence. For Spaceship, it is, in fact, the only algorithm that enabled convergence.  In contrast, without the safety layer, a significant amount of episodes ended with a constraint violation and convergence was often not attained. 
   This is due to the nature of our tasks: frequent episode terminations upon boundary crossing impede the learning process in our sparse reward environments. However, with the safety layer, these terminations never occur, allowing the agent to maneuver as if there were no boundaries.
   Next, the following domain-specific discussion is in order.

   	In the Ball domain, since in the 1D task the target is always located in one of two possible directions, its reward is less sparse compared to 3D. This easy setup allowed DDPG to converge to a reasonable return value even without reward shaping or the safety layer; in 3D this happened only for the upper quantile of the seeds. In both tasks, an improvement was obtained with reward shaping. Nonetheless, DDPG+safety layer obtained the highest discounted return with much faster convergence. As for cumulative constraint violations, as DDPG converged slower than DDPG+reward shaping, it also stabilized later on a higher value. All the same, DDPG+safety layer accomplished our safety goal and maintained $0$ accumulated constraint violations.
   	
   	In the Spaceship domain, the reward is obtained only once at the target, and the spaceship is re-initialized away from it with each constraint violation. This setup proves fatal for DDPG, which was not able to converge to any reasonable policy in both tasks. Surprisingly, reward shaping poses no improvement but rather has an adverse effect: it resulted in highly negative episodic returns. On the other hand, DDPG+safety layer converged extremely fast to a high-performing safe policy. This behavior stems from the closed-region type of tasks; exploring while straying away from the walls allowed the spaceship to quickly meet the target and then learn how to reach it. With regards to safety, DDPG+safety layer again prevailed where DDPG and DDPG+reward shaping failed, and maintained $0$ accumulated constraint violations.

   \subsection{Videos}
   To visualize the safety layer in action, we depict the intensity of its action correction via the magnitude of the dominant Lagrange multiplier $\lambda^*_{i^*}$ as calculated in \eqref{eq:optimal multiplier}. We do so by coloring the Ball and Spaceship objects in varying shades of red. We share two videos of the colored agents showcasing the following.
   
   	 In Video~1, to track the learning process of an agent with the safety layer, we recorded several episodes during training in the Spaceship-Corridor task. These exhibit how the maneuvers produced by the safety layer to dodge the walls are gradually learned by the agent, as demonstrated by the less frequent coloring of the spaceship red. Video~1 is found in    \url{https://youtu.be/KgMvxVST-9U}.
 	 	 In Video~2, per each task of the four, we show episodes of i) the first DDPG iteration (initialized with a random policy) with the safety layer off; ii) the same random initial policy but with the safety layer on; and iii) the final iteration of DDPG+safety layer, to demonstrate the task's goal. Video~2 is found in       \url{https://youtu.be/yr6y4Mb1ktI}.

   \section{Discussion}

   	In this work, we proposed a state-based action correction mechanism, which accomplishes the goal of zero-constraint-violations in tasks where the agent is constrained to a confined region. This is in contrast with the standard reward shaping alternative, which has failed in achieving the above goal. 
   	The resulting gain is not only in maintaining safety but also in enhanced performance in terms of reward. This suggests our method promotes more efficient exploration -- it guides the exploratory actions in the direction of feasible policies.
   	Since our solution is stand-alone and applied directly at the policy level, it is independent of the RL algorithm used 
   	and can be plugged into any other continuous control algorithm.
   	
   	Throughout this work, we relate to a preemption mechanism implemented in real-world critical systems, which halts the RL agent in borderline situations and replaces it with a safe-operation heuristic. The latter runs until the system is back to being far from operation limits. This heuristic is expected to be conservative and less efficient than the RL agent; hence, the contribution of our work can also be interpreted as in reducing operational costs by minimizing the number of times such takeovers occur.
   	
   	Lastly, an advantage of our approach over off-policy methods, often considered in industry, is that one does not need to know the behavior policy used to generate existing data logs. This is thanks to our single-step model, which we train on trajectories generated with random actions. These trajectories are always within operating limits (due to episode termination when the limits are crossed) and are independent of any particular policy whose long-term behavior can be elevated. Nevertheless, they carry rich information due to their exploratory nature. It is thus intriguing to study in future work additional types of pre-training data that are typical to specific real-world domains.


   \label{sec:discussion}

\bibliography{safety_layer}
\bibliographystyle{icml2017}

\end{document}